\documentclass{article}
\usepackage{arxiv}

\usepackage[utf8]{inputenc} 
\usepackage[T1]{fontenc}    
\usepackage{hyperref}       
\usepackage{url}            
\usepackage{booktabs}       
\usepackage{amsfonts}       
\usepackage{nicefrac}       
\usepackage{microtype}      
\usepackage{paralist,multirow,graphicx,amsmath,amsfonts,amssymb,epstopdf,color,epsfig,ragged2e,array,standalone,booktabs,subfigure}
\usepackage{stfloats}
\graphicspath{ {./images/} }

\usepackage{array}
\usepackage{tabularx}
\usepackage{amsmath}
\usepackage{amsthm}
\usepackage{amssymb}
\usepackage{amsfonts}
\usepackage{mathtools}
\usepackage{algorithm}
\usepackage{algorithmic}
\usepackage{enumitem}
\usepackage{centernot}
\usepackage[toc,page]{appendix}

\DeclarePairedDelimiter\floor{\lfloor}{\rfloor}
\DeclareMathOperator\supp{supp}

\newcommand{\reals}{\ensuremath{\mathbb{R}}}
\newcommand{\naturals}{\ensuremath{\mathbb{N}}}

\newcommand{\defeq}{\vcentcolon=}

\newcommand{\prob}{\ensuremath{\mathbb{P}}}

\newcommand{\va}{\mathbf a}

\newcommand{\vx}{\mathbf x}

\newcommand{\vv}{\mathbf v}

\newcommand{\mB}{\ensuremath{\mathbf{B}}}

\newcommand{\mA}{\ensuremath{\mathbf{A}}}

\newtheorem{theorem}{Theorem}[section]

\newtheorem{lemma}[theorem]{Lemma}
\newtheorem{definition}{Definition}[section]

\title{Analysing recovery of activation pathways in DCNNs via Deep Convolutional Sparse Coding}

\author{
 Michael Murray, Jared Tanner\\
  Mathematical Institute, University of Oxford, UK\\
  \& The Alan Turing Institute, London, UK\\
  \texttt{murray@maths.ox.ac.uk}\\
  }

\begin{document}
\maketitle
\begin{abstract}
In this paper we investigate the impact of ReLU and related sparsifying activation functions on signal propagation in the forward pass of Deep Convolutional Neural Networks (DCNNs) \footnote{This paper is a long print version of a paper accepted to the 2018 IEEE Data Science Workshop.}. To this end, we consider a variant of the approach proposed by Papyan et al \cite{JMLR:v18:16-505}, which interprets the forward pass of a DCNN as solving a sequence of sparse coding problems, we therefore refer to this approach as Deep Convolutional Sparse Coding (DCSC). In \cite{JMLR:v18:16-505} the authors proved that representations with an activation density proportional to the ambient dimension of the data are at least approximately recoverable. We extend these uniform guarantees and prove with high probability that representations with a far greater density of activations per layer are approximately recoverable.
\end{abstract}
\section{Introduction}
Ever since the arrival of AlexNet \cite{NIPS2012_4824} in 2012, DCCNs have been the state of the art for many problems in computer vision. They have also achieved excellent results in a host of other applications, including Natural Language Processing \cite{Kim2014ConvolutionalNN} and Speech Recognition \cite{2017arXiv170102720Z}. In this paper we analyse signal propagation in the forward pass of Deep Convolutional Neural Networks (DCNNs), seeking to better understand the role of the activation function in computing representations with explanatory power. To this end we build on the work of Papyan et al \cite{JMLR:v18:16-505}, who, inspired by the connections between convolutional weight matrices used in deep learning and the dictionaries used in convolutional sparse coding \cite{2017ITSP...65.5687P}, as well as the fact that ReLU activation functions are sparsifying, interpreted the forward pass of a DCNN as solving a sequence of convolutional sparse coding problems. This interpretation enables the analysis of the sequence of representations generated at each layer of the network using tools and ideas from compressed sensing.

To this end, we introduce and study the Deep Convolutional Sparse Coding (DCSC) model, defined in Definition \ref{ch3_def_DCSC_data_model}. This model assumes that the data is generated, at least approximately, by a matrix product between a dictionary, referred to as the global dictionary, and a sparse latent representation. Of key importance is the factorised form of this global dictionary and the sparse intermediary representations generated at different levels or layers of this factorisation. This data encoder can therefore be viewed as a linear network, with the activation pathway of a data point being the set of neurons at each layer which fire, i.e., are nonzero, as the signal propagates from the latent space to the observed data space. The activation pathway of a data point therefore indicates the key features present in the data and hence has strong explanatory power. In the DCSC model the forward pass of a DCNN is interpreted as the decoder associated with this linear encoder. The learning of the weights of the decoder is omitted and the parameters of the encoder and decoder are shared: this allows us to more transparently analyse the role of the activation function in enabling the forward pass to recover the activation pathway of a data point.

Papyan et al \cite{JMLR:v18:16-505} conducted an analysis of a similar model, demonstrating its connection to DCNNs and proved conditions under which the forward pass is guaranteed to recover activation pathways. A technical innovation of their work highlights that one can measure the efficacy of a sparsifying activation function through a new, local measure of sparsity particular to the convolutional structure present, referred to as stripe-sparsity. Using this measure the authors proved that representations with an activation density proportional to the ambient dimension of the data are recoverable. However, the upper bounds derived in \cite{JMLR:v18:16-505} on the stripe-sparsity of a recoverable activation at a given layer depend on the inverse of the mutual coherence of the weight matrix at that layer, which is typically quite small. This limits the applicability of these results as only data points with a very sparse activation near the input of the encoder are recoverable by the forward pass decoder. In this paper we extend these uniform guarantees to the modified DCSC model and prove that activation pathways with a greater density of activations per layer are recoverable with high probability. To prove this result we leverage techniques based on one step thresholding developed by Schnass and Vandergheynst \cite{4351958}.
\section{Inference in deep learning as sequential sparse coding} \label{ch3_sec_problem_setup}
\subsection{The Deep Convolutional Sparse Coding (DCSC) model} \label{ch3_subsec_DCSC_model}
We now introduce and define the DCSC model,  a particular instance of a weight tied encoder-decoder neural network pair based on the DCP model \cite{JMLR:v18:16-505}.

\begin{definition}[DCSC data model] \label{ch3_def_DCSC_data_model}
	 In order to define the DCSC encoder we introduce the following variables.
	\begin{itemize}
		\item $\textbf{A}^{(l)} \in \reals^{n_{l-1}M \times n_{l}M}$ is a deterministic convolutional matrix (see \cite{JMLR:v18:16-505} and \cite{2017ITSP...65.5687P} for further details) which is circular, banded and created by shifting a local dictionary $\textbf{A}_{Local}^{(l)} \in \reals^{m_l \times n_l}$ across all spatial locations.  At each layer we interpret $n_l$ as the number of local filters and $m_l$ as the dimension of each local filter. We further specify that $n_0 \defeq 1$, $m_l \defeq n_{l-1}m_{l-1}$ and for $l \geq 2$ there is a stride $s_{l} = n_{l-1}$ between each spatially shifted $\textbf{A}^{(l)}_{local}$. The columns of $\textbf{A}^{(l)}$ are assumed to have unit $\ell_2$ norm and are denoted as $\textbf{A}^{(l)}  =[\textbf{a}_1^{(l)} \ \textbf{a}_2^{(l)} \ ... \ \textbf{a}_{n_l M}^{(l)}] $.
		\item $D^{(l)}$ is a random, square, binary, diagonal matrix of size $n_{l}M \times n_{l}M $ whose diagonal entries are independent and identically distributed, taking values in $\{-1, 1\}$ each with probability $0.5$.
	\end{itemize}
		
	\noindent The DCSC encoder is a function $E_{DCSC}: \reals^{n_L M} \rightarrow \reals^{M }$ parameterized by the forward pass of a linear neural network. The input to the encoder is denoted $\vx^{(L)} \in \reals^{n_lM}$ and the representations of $\vx^{(L)}$ generated at each layer of the DCSC encoder are given by the recurrence relation
	\[
	\begin{aligned}
	\vx^{(l-1)} &\defeq \mA^{(l)} D^{(l)} \vx^{(l)} \;\; \forall \;\; l \in[L].
	\end{aligned}
	\]	
	Given a DCSC encoder, the corresponding weight tied decoder function $D_{DCSC}: \reals^{M } \rightarrow \reals^{n_LM}$ is parameterized by the forward pass of another neural network equipped with a nonlinear, sparsifying activation function. The input to the decoder is defined as
	\[
	\hat{\vx}^{(0)} \defeq E_{DCSC}(\vx^{(L)}) + \vv^{(0)} = \vx^{(0)} + \vv^{(0)},
	\]
	where $\vv^{(0)} \in \reals^{M}$ denotes the model noise. The representations of $\hat{\vx}^{(0)}$  generated at each layer of the decoder are defined recursively for $l \in [L]$ as
	\begin{equation}
	\hat{\vx}^{(l)} \defeq Proj_{|\supp(\cdot)|=k}\left( (\textbf{A}^{(l)}D^{(l)})^T \hat{\vx}^{(l-1)} \right),
	\end{equation} \label{eq:ch3_projection}
	where the projection operator $Proj_{|\supp(\cdot)|=k}(\cdot)$ keeps the $k$ largest elements in terms of absolute value unchanged and sets all other elements to zero. We further define the representation error between the encoder and decoder at the $l$th layer as
	\[
	\vv^{(l)} \defeq \vx^{(l)} - \hat{\vx}^{(l)}.
	\]
\end{definition}

In regard to Definition \ref{ch3_def_DCSC_data_model} a few remarks are in order. First, by replacing $D^{(l)}$ with the identity matrix at each layer $l \in [L]$, then the $DCP_{\lambda}$ and $DCP_{\lambda}^{\epsilon}$ models presented in \cite{JMLR:v18:16-505} are recovered. Second, by substituting $Proj_{|\supp(\cdot)|=k}(\cdot)$ with a ReLU operator we obtain the standard forward pass algorithm across a ReLU layer of a neural network with $(\textbf{A}^{(l)}\textbf{D}^{(l)})^T$ the weight matrix between the $l-1$th and $l$th layers. One can interpret $Proj_{|\supp(\cdot)|=k}(\cdot)$ as model for a family of sparsifying operators, of which ReLU and the soft and hard thresholding operators are examples. Indeed, in almost any practical circumstance, by adjusting the bias at each layer appropriately it should be clear how each specific sparsifying activation function can implement a projection onto the $k$ largest entries of the argument vector in question.

There are numerous questions one might ask concerning the DCSC model, for instance, what conditions are sufficient and or necessary for recovery in the noiseless case, or to ensure$\|\vx^{(l)} - \hat{\vx}^{(l)} \|_2 \leq \epsilon$ for all $l\in [L]$ and some $\epsilon \in \reals_{>0}$. These are some of the questions studied in \cite{JMLR:v18:16-505} in the case where $D^{(l)}$ is the identity matrix. Also covered in this work is a study on the recovery of the support at each layer, it is this notion of recovery that we will focus on.

\begin{definition}[Activation pathway]
	The activation pathway associated with $\vx^{(L)} \in \reals^{n_L M}$ is the sequence of supports $\{\supp(\vx^{(l)})\}_{l=1}^L$. We say that the activation pathway of $\vx^{(L)} $ is recovered up to layer $l$ iff $\supp(\hat{\vx}^{(k)}) = \supp(\vx^{(k)})$ for all $k \in [l]$. An activation pathway is recovered iff it is recovered at all layers, i.e, up to layer $L$.
\end{definition}

\noindent To motivate why we are interested in the recovery of activation pathways, we note that in many applications we are not necessarily interested in achieving perfect reconstruction. Instead, we wish to ensure that a trained network is able to identify the key features of the observed data. Recovery of the activation pathway implies, by construction, that the decoder identifies the presence of the salient features of the data used in its creation.

\subsection{Uniform guarantees for activation pathway recovery} \label{ch3_subsec_unifrom guarantees}
\noindent In \cite{JMLR:v18:16-505}, and under the assumption that $D^{(l)}$ is the identity matrix for all $l \in [L]$,  Papyan, Romano, and Elad studied the $\{ \hat{\vx}^{(l)} \}_{l=1}^L$ obtained by the DCSC encoder in relation to the $\{\vx^{(l)}\}_{l=1}^L$ generated by the encoder. In particular, they proved various forms of recovery guarantees under sparsity constraints on the encoder representations. Their analysis relies heavily on the notion of the coherence of a dictionary, 
\begin{equation}\label{coherence}
\mu(\textbf{A}) \defeq \max_{i\ne j} | \langle\textbf{a}_i,\textbf{a}_j \rangle|,
\end{equation} 
where $\textbf{a}_i$ is the $i^{th}$ column of $\textbf{A}$. One of the main technical innovations in \cite{JMLR:v18:16-505} was the derivation of traditional sparse approximation bounds in the convolutional, multilayer setting. To this end they introduced the following novel local sparsity measures, based on the banded, circular structure of each $\mA^{(l)}$, and used them to ameliorate the limited lower bound on \eqref{coherence}. 
\begin{itemize}
	\item $\| \textbf{x} \|_{\alpha, \infty}^{P^{(l)}} \defeq \max_{i \in n_lM} \| P^{(l)}(i)\textbf{x} \|_\alpha$ where $P^{(l)}(i)$, the patch operator at the $l$th layer, is an $n_lM \times n_lM$ diagonal, binary matrix with exactly $m_{l+1}$ consecutive nonzeros starting at row $i$ with wraparound. To be clear, if $i+ m_l -1 \leq n_lM$ then $P_{j,j}^{(l)}(i) = 1$ iff $i\leq j\leq i+ m_l -1$, if $i+ m_l -1 >n_lM$ then $P_{j,j}^{(l)} (i)= 1$ iff $i\leq j\leq n_lM$ or $1\leq j \leq m_l - 1 - (n_l M - i) $  (see \cite{JMLR:v18:16-505} for further details). In this paper we will only consider $\alpha \in \{0,2\}$, hence $\| \cdot \|_{\alpha}$ refers to the euclidean norm when $\alpha = 2$, and a function counting the number of non-zeros in the argument vector when $\alpha = 0$.
	\item $\| \textbf{x} \|_{\alpha, \infty}^{Q^{(l)}} \defeq \max_{i \in n_lM} \| Q_i^{(l)}\textbf{x} \|_\alpha$ where $Q^{(l)}(i)$, the stripe operator, is an $n_lM \times n_lM$ diagonal, binary matrix with exactly $\floor{((2(m_l/s_{l})-1)n_l)}$ consecutive nonzeros starting at row $i$ with wraparound. To be clear, if $i+ \floor{((2(m_l/s_{l})-1)n_l)} -1 \leq n_lM$ then $Q_{j,j}^{(l)}(i) = 1$ iff $i\leq j\leq i+ \floor{((2(m_l/s_{l})-1)n_l)} -1$, if $i+ \floor{((2(m_l/s_{l})-1)n_l)} -1 >n_lM$ then $Q_{j,j}^{(l)} (i)= 1$ iff $i\leq j\leq n_lM$ or $1\leq j \leq \floor{((2(m_l/s_{l})-1)n_l)} - 1 - (n_l M - i) $ (again see \cite{JMLR:v18:16-505} for further details). A stripe of $\vx^{(l)}$ then is the sparse code associated with a particular patch of $\vx^{(l-1)}$. As before we will only consider $\alpha \in \{0,2\}$.
\end{itemize}

Papyan et al's analysis is wide-ranging, including sparsity conditions under which the representations generated by the decoder are unique. They also consider a variety of thresholding operators such as soft and hard thresholding as well as more advanced algorithms to compute $\hat{\vx}^{(l)}$ from $\textbf{A}^{(l)}$ and $\hat{\vx}^{(l-1)}$. We focus only on the derived uniform bound concerning the recovery of activation pathways. Assume that $\supp(\hat{\vx}^{(k)}) = \supp(\vx^{(k)})$ for all $k < l$, and that the cardinality of $\supp(\vx^{(l)})$ is known for all $l \in [L]$ by the decoder network.  Papyan et al proved that as long as 
\begin{equation}\label{mu_inv}
\| \vx^{(l)} \|_{0,
	\infty}^{Q^{(l)}}<\frac{1}{\mu^{(l)}|x^{(l)}_{max}|} \left(\frac{1}{2}|x^{(l)}_{min}|-\zeta_l\right) +\frac{1}{2},
\end{equation}
where $\zeta_l \geq \| \vv^{(l)} \|_{2, \infty}^{P^{(l)}}$ is an upper bound on the patch error at the $l$th layer, $\mu_{l} \defeq \mu(\textbf{A}^{(l)})$ and $|x^{(l)}_{min}|$ and $|x^{(l)}_{max}|$ are the smallest and largest non-zero entries of $\vx^{(l)}$ respectively, then $\supp(\hat{\vx}^{(l)}) = \supp(\vx^{(l)})$.

\section{Recovery of denser activation pathways} \label{ch3_sec_analysis}
Notable in the sparsity bound \eqref{mu_inv} is the presence of $\mu_{l}$, which allows for a nontrivial stripe sparsity.  Bounds of the form \eqref{mu_inv} are prevalent in the theory of sparse approximation, see for instance \cite[Chapter 5]{intro_CS}, where it is known \cite{articleWelch}
for a generic matrix $\mB\in \reals^{m\times \gamma m}$ that $\mu(\mB)> m^{-1/2}\sqrt{1-\gamma^{-1}}$. This is colloquially referred to as the square-root bottleneck in that $\mu^{-1}\sim m^{1/2}$. In many applications, e.g. imaging, typically $m_l$ is not more than $7^2$ and $n_l$ is approximately $2m_l$. In addition, guaranteeing the recovery of denser activations is also made challenging due to the fact that $\mA^{(l)}$ is a convolutional matrix. This structure can result in a large mutual coherence if the stride between shifted versions of the local dictionary $\mA_{Local}^{(l)}$ is small. As a result, the proportionality of $\mu_l$ to the signal complexity, measured in terms of the sparsity, limits the ability of this prior work to provide guarantees in many practical situations.

It is well known from the work of Schnass and Vandergheynst \cite{4351958} that, in the single layer context, if one introduces a randomised sign pattern then a Rademacher concentration inequality can be used to derive bounds demonstrating that the recovery of activations is typically possible even when the sparsity constraint is relaxed to depend on $\mu_l^{-2}$.

\begin{theorem}[Rademacher concentration \cite{Talagrand}]
	\label{theorem:rademacher_cm_inequality}
	Let $\alpha$ be an arbitrary real vector and
	$\varepsilon$ a random vector whose elements are
	independent Rademacher
	random variables. Then for all $t \in \mathbb{R}_{>0}$
	\begin{equation}
	P \left( |\sum_{i}  \varepsilon_i \alpha_i |>t \right) \leq 2\exp \left( -  \frac{t^2}{2\| \alpha \|_2^2} \right).
	\end{equation}
\end{theorem}

Our main contribution is to extend the techniques used in \cite{4351958} to the multi-layer setting of \cite{JMLR:v18:16-505}, which explains and motivates the introduction of the random diagonal matrix $D^{(l)}$ at each layer of the DCSC model. This matrix applies a random sign pattern to the columns of $\textbf{A}^{(l)}$ and although this matrix is primarily an artefact necessary for our analysis, it is interesting to note its connection with dropout. Dropout is a technique commonly used when training DCNNs in which a random set of nodes (or columns of the weight matrix) are ignored in every update of the weights. Indeed, one can tentatively interpret $D^{(l)}$ as a special form of dropout, which selects either the positive or negative signed column from a wider dictionary that contains both. Under this adaption, and recalling that $|x^{(l)}_{min}|$ and $|x^{(l)}_{max}|$ are the smallest and largest non-zeros in terms of absolute value of $\vx^{(l)}$, then we are able to provide Theorem \ref{ch3_thm_prob_bound}.

\begin{theorem}
	\label{ch3_thm_prob_bound}
	Under the DCSC model,  for each $l \in [L]$ let $S_l\in \naturals$ 
	be an upper bound on the stripe sparsity, $\| \vx^{(l)} \|_{0, \infty}^{Q^{(l)}}  \leq S_l$. Assume that the model noise $\vv^{(0)}$ is such that $\supp(\hat{\vx}^{(0)}) = \supp(\vx^{(0)})$. Furthermore, for each $l \in [L] \cup \{0\}$ let $\zeta_l \in \reals_{>0}$ be an upper bound on the patch error, $\| \vv^{(l)} \|_{2, \infty}^{P^{(l)}}  \leq \zeta_l $. Then the probability that the activation pathway of $\vx^{(L)}$ is recovered is at least
	\[
	1 -  2M \sum_{l=1}^L n_l \exp \left( -  \frac{|x^{(l)}_{min}|^2}{8  \left(|x^{(l)}_{max}|^2\mu_l^2S_l +\zeta_{l-1}^2 \right) } \right).
	\]
	In addition, assuming that $\supp(\hat{\vx}^{(l)}) = \supp(\vx^{(l)})$ and defining $\zeta_0 \defeq || \vv^{(0)}||_{2, \infty}^{P^{(l)}}$, then
	\[
	\zeta_{l} = \sqrt{\|  \hat{\textbf{x}}^{(l)}  \|_{0, \infty}^{P^{(l)}} } \left( \mu_l( S_l-1)|x_{max}^{(l)} | + \zeta_{l-1} \right)
	\] 
	is a valid upper bound on the error $|| \vv^{(l)}||_{2, \infty}^{P^{(l)}}\leq \zeta_l$ for each layer $l \in [L]$.
\end{theorem}

\noindent A key implication of Theorem \ref{ch3_thm_prob_bound} is that the bound on the density of nonzeros scales proportional to $ \mu_l^{-2}$ across a given layer rather than $\mu_l^{-1}$. Assume that $\supp(\hat{\vx}^{(k)}) = \supp(\vx^{(k)})$ for all $k < l$, and that the cardinality of $\supp(\vx^{(l)})$ is known for all $l \in [L]$ by the decoder network. With $\delta \in \left(2n_lM \exp \left( -  \frac{|x^{(l)}_{min}|^2}{8  \left(|x^{(l)}_{max}|^2\mu_l^2S_l +\zeta_{l-1}^2 \right) } \right),1 \right]$ (we refer the reader to Lemma \ref{ch3_lemma_single_layer} for details) then as long as
\begin{equation}
S_l \leq  \left( \frac{|x_{min}^{(l)}|^2}{8|x_{max}^{(l)}|^2\ln\left(\frac{2Mn_l}{\delta}\right)}- \frac{\zeta_{l-1}^2}{|x_{max}^{(l)}|^2} \right)\mu_{l}^{-2} 
\end{equation}
then the probability that the activation pattern at the $lth$ layer is recovered is at least $1 - \delta$.

\noindent We develop a proof of Theorem \ref{ch3_thm_prob_bound} using induction, analysing the probability that the forward pass fails to recover the activation pathway at an arbitrary layer $l\in [L]$ conditioned on recovery up to layer $l-1$. To this end we provide Lemma \ref{ch3_lemma_single_layer}, which extends bounds provided in \cite{4351958} to also include additive noise and the notion of local stripe sparsity.

\begin{lemma} \label{ch3_lemma_single_layer}
	Under the DCSC model, for each $l \in [L]$ let $S_l\in \naturals$ be an upper bound on the stripe sparsity $\| \vx^{(l)} \|_{0, \infty}^{Q^{(l)}}  \leq S_l$ of the encoder representation at the $l$th layer. Suppose for some $l\in [L]$ that  $\supp(\hat{\vx}^{(l-1)}) = \supp(\vx^{(l-1)})$ and that $ \zeta_{l-1}\geq \| \vv^{(l-1)} \|_{2, \infty}^{P^{(l-1)}} $. Then the probability that $\supp(\hat{\vx}^{(l)}) \neq \supp(\vx^{(l)})$ is at most
	\[
	2n_lM \exp \left( -  \frac{|x^{(l)}_{min}|^2}{8  \left(|x^{(l)}_{max}|^2\mu_l^2S_l +\zeta_{l-1}^2 \right) } \right).
	\]
	If $\supp(\hat{\vx}^{(l)}) = \supp(\vx^{(l)})$ then a valid upper bound for the patch error $|| \vv^{(l)}||_{2,\infty}^{P^{(l)}} $ is
	\[
	\zeta_{l} = \sqrt{\|  \hat{\textbf{x}}^{(l)}  \|_{0, \infty}^{P^{(l)}} } \left( \mu_l( S_l-1)|x_{max}^{(l)} | + \zeta_{l-1} \right).
	\] 
\end{lemma}

\begin{proof}	
	First, and for typographical ease, we drop the $l$ superscript on both $\mA^{(l)}$ and $D^{(l)}$. We further denote the $j$th diagonal element of $D$ as $\varepsilon_j$, and recall that these are mutually independent random variables with value either $-1$ or $1$ both with probability $0.5$. Furthermore, and again for notational ease, we define $\Lambda \defeq \supp(\vx^{(l)})$. A superscript bar will be used to denote the compliment of a set, for example $\bar{\Lambda} = [n_l M] \backslash \Lambda$. The event that the DCSC decoder recovers the support of the encoder representation, i.e., $\supp(\hat{\vx}^{(l)}) =\supp(\vx^{(l)})$,  will be denoted $W^{(l)}$, and, in keeping with the other notational aspects just mentioned, $\bar{W}^{(l)}$ will denote the event that $\supp(\hat{\vx}^{(l)}) \neq \supp(\vx^{(l)})$. Finally, due to the presence of various superscripts, we will use $\langle \cdot, \cdot\rangle: \reals^{n_l M} \times \reals^{n_l M} \rightarrow \reals$ to refer to the Euclidean inner product or dot product on $\reals^{n_l M}$.
	
	Considering Equation \eqref{eq:ch3_projection}, then for the DCSC decoder to fail to recover the support of the encoder representation, there must exist a nonzero entry in the decoder representation which is not in the support of the encoder representation and whose magnitude is larger than at least one of the nonzeros in the encoder representation. This means that there exists some $i \in \Lambda$ and some $k \in \bar{\Lambda}$ such that 
	\[
	| \langle \varepsilon_i \textbf{a}_i, \hat{\textbf{x}}^{(l-1)} \rangle | < | \langle \varepsilon_k \textbf{a}_k,  \hat{\textbf{x}}^{(l-1)} \rangle| .
	\]
	This condition is equivalent to requiring
	\[
	\min_{i \in \Lambda} | \langle \varepsilon_i \textbf{a}_i,  \hat{\textbf{x}}^{(l-1)}\rangle |< \max_{k \in \bar{\Lambda}}| \langle \varepsilon_k \textbf{a}_k,  \hat{\textbf{x}}^{(l-1)}\rangle |
	\]
	and therefore
	\[
	\prob(\bar{W}^{(l)}) = \prob(\min_{i \in \Lambda} | \langle \varepsilon_i \textbf{a}_i,  \hat{\textbf{x}}^{(l-1)} \rangle | < \max_{k \in \bar{\Lambda}} | \langle \varepsilon_k \textbf{a}_k,  \hat{\textbf{x}}^{(l-1)} \rangle | ).
	\]
	For an arbitrary $p \in \reals$, if $\min_{i \in \Lambda} | \langle \varepsilon_i \textbf{a}_i,  \hat{\textbf{x}}^{(l-1)}\rangle |< \max_{k \in \bar{\Lambda}}| \langle \varepsilon_k \textbf{a}_k,  \hat{\textbf{x}}^{(l-1)} \rangle | $ holds true then the event $\{\min_{i \in \Lambda} | \langle \varepsilon_i \textbf{a}_i,  \hat{\textbf{x}}^{(l-1)} \rangle | < p \} \cup \{ \max_{k \in \bar{\Lambda}} | \langle \varepsilon_k \textbf{a}_k,  \hat{\textbf{x}}^{(l-1)} \rangle | > p \}$ is also true. Applying the union bound it therefore follows that
	\begin{equation}
	\begin{split}
	\prob(\bar{W}^{(l)}) \leq&  \prob(\min_{i \in \Lambda} | \langle \varepsilon_i \textbf{a}_i,  \hat{\textbf{x}}^{(l-1)} \rangle | < p) + \prob(\max_{k \in \bar{\Lambda}} | \langle \varepsilon_k \textbf{a}_k,  \hat{\textbf{x}}^{(l-1)} \rangle | > p ).
	\end{split}
	\end{equation} \label{eq:ch3_split}
	We now provide bounds on each of the terms on the right hand side of the above inequality using the
	Rademacher concentration inequality stated in Theorem \ref{theorem:rademacher_cm_inequality}. Considering first the second term,
	\begin{equation}
	\begin{aligned}
	P\left( \max_{k \in \bar{\Lambda}} | \langle  \varepsilon_k  \textbf{a}_k, \hat{\textbf{x}}^{(l-1)} \rangle |> p \right)
	& \leq  \sum_{k\in \bar{\Lambda}} P\left(| \langle  \varepsilon_k \textbf{a}_k, \hat{\textbf{x}}^{(l-1)} \rangle | > p \right) \label{W} \nonumber\\
	& = \sum_{k\in \bar{\Lambda}} P \left( |\sum_{j \in \Lambda}
	\varepsilon_j' x_j ^{(l)}\langle  \textbf{a}_k, \textbf{a}_j
	\rangle + \varepsilon_k \langle   \textbf{a}_k, \vv^{(l-1)}
	\rangle | >p  \right)\\
	& \leq  2\sum_{k\in \bar{\Lambda}} \exp \left( \frac{-p^2}{2
		\left( \sum_{j \in \Lambda \cap \Gamma} |x^{(l)}_j|^2 |
		\langle  \textbf{a}_k, \textbf{a}_j \rangle|^2
		+\zeta_{l-1}^2 \right) } \right) \nonumber \\
	& \leq 2 (n_lM-|\Lambda|) \exp \left( \frac{-p^2}{2 \left(
		|x^{(l)}_{max}|^2 S_l\mu_l^2+\zeta_{l-1}^2 \right) }
	\right) \nonumber.
	\end{aligned}
	\end{equation}
	
	\noindent The first line and inequality arises from $\max_{k \in \bar{\Lambda}} \{ | \langle  \varepsilon_k  \textbf{a}_k, \hat{\textbf{x}}^{(l-1)} \rangle | > p\}$ implying that\\
	$\cup_{k \in \bar{\Lambda}} \{| \langle \varepsilon_k \textbf{a}_k, \hat{\textbf{x}}^{(l-1)} \rangle |\} > p$ and then applying the union bound. The second line is an expansion of the inner product using $\hat{\vx}^{l-1} = \vx^{l-1} + \vv^{(l-1)} =   \mA^{(l)}D^{(l)}\vx^{l} + \vv^{(l-1)}$. Here we also introduce a new Rademacher random variable $\varepsilon_j' \defeq \varepsilon_j\varepsilon_k$ and note that the set of random variables $\left( \bigcup_ {j \in n_lM} \{\varepsilon_j'\}\right) \cup \{\varepsilon_k\}$ are mutually independent. Moving from the second to the third line, we use Theorem \ref{theorem:rademacher_cm_inequality} and introduce the set $\Gamma$ to denote the indices of columns of $\textbf{A}^{(l)}$ which have a nonzero inner product with the column $\textbf{a}_k$. Furthermore, as $\va_k$ has unit $\ell_2$ norm and $|\supp(\va_k)| = m_l$, then $|\langle \va_i, \vv^{(l-1)} \rangle|  \leq \zeta_{l-1}^2$ by construction. The final line then follows from the fact that $|\langle a_k, a_j \rangle|\leq \mu_l^2$ for any $j \neq k$ and $|\Lambda \cap \Gamma| \leq S_l$, which in turn is a consequence of the assumption that $\| \textbf{x}^{(l)} \|_{0,\infty}^{Q^{(l)}}\le S_l$.
	
	Turning our attention to bounding the probability of $\min_{i \in \Lambda} | \langle \varepsilon_i \textbf{a}_i,  \hat{\textbf{x}}^{(l-1)} \rangle | < p$, we first expand the inner product as before and then use the triangle inequality to conclude that
	\[
	| \langle \varepsilon_i \textbf{a}_i,  \hat{\textbf{x}}^{(l-1)} \rangle |
	\geq | x_i| - \left| \sum_{j \in \Lambda, j \neq i} \varepsilon_j' x_j^{(l)}  \langle  \textbf{a}_i, \textbf{a}_j \rangle + \varepsilon_i \langle  \textbf{a}_i, \textbf{v}^{(l-1)} \rangle \right|.
	\]
	We are then able to bound the probability that $\min_{i \in \Lambda} | \langle \varepsilon_i \textbf{a}_i,  \hat{\textbf{x}}^{(l-1)} \rangle | < p$ using the same steps as before for $ \max_{k \in \bar{\Lambda}} \{| \langle  \varepsilon_k  \textbf{a}_k, \hat{\textbf{x}}^{(l-1)} \rangle |\} > p$.
	\[
	\begin{split}
	\prob(\min_{i \in \Lambda} | \langle \varepsilon_i \textbf{a}_i,  \hat{\textbf{x}}^{(l-1)} \rangle | < p) & \leq P \left( \max_{i \in \Lambda} | \sum_{j \in \Lambda, j \neq i}\varepsilon_j' x_j^{(l)}  \langle  \textbf{a}_i, \textbf{a}_j \rangle + \varepsilon_i \langle  \textbf{a}_i, \textbf{v}^{(l-1)} \rangle| > |x^{(l)}_{min}| - p \right) \\
	&\leq \sum_{i \in \Lambda} P\left( |\sum_{j \in \Lambda, j \neq i}\varepsilon_j' x_j^{(l)}  \langle  \textbf{a}_i, \textbf{a}_j \rangle + \varepsilon_i \langle  \textbf{a}_i, \textbf{v}^{(l-1)} \rangle| > |x^{(l)}_{min}| - p \right) \\
	& \leq 2 \sum_{i \in \Lambda}  \exp \left( -  \frac{(|x_{min}^{(l)}| - p)^2}{2  \left( \sum_{j \in \Lambda \cap \Gamma / i} |x_j|^2 | \langle  \textbf{a}_k, \textbf{a}_j \rangle|^2 + \zeta_{l-1}^2\right) } \right) \\
	& \leq 2|\Lambda| \exp \left( \frac{-(|x_{min}^{(l)}| - p)^2}{2 \left( |x^{(l)}_{max}|^2 S_l\mu_l^2+\zeta_{l-1}^2 \right) } \right).
	\end{split}
	\]
	The first line is a result of rearranging and bounding the expanded inner product, the subsequent lines then follow in the same manner as for $ \max_{k \in \bar{\Lambda}} \{| \langle  \varepsilon_k  \textbf{a}_k, \hat{\textbf{x}}^{(l-1)} \rangle |\} > p$.  Recalling that $p \in \reals_{\geq 0}$ is arbitrary, then to recover the bound claimed we let $p  = |x_{min}| /2 $. Indeed, for this value of $p$ it follows that
	\[
	\begin{aligned}
	\prob(\bar{W}^{(l)}) &\leq   2 (n_l M-|\Lambda|) \exp \left( \frac{-p^2}{2 \left(
		|x^{(l)}_{max}|^2 S_l\mu_l^2+\zeta_{l-1}^2 \right) }
	\right)  + 2|\Lambda| \exp \left( \frac{-(|x_{min}^{(l)}| - p)^2}{2 \left( |x^{(l)}_{max}|^2 S_l\mu_l^2+\zeta_{l-1}^2 \right) } \right) \\
	& = 2 n_l M \exp \left( -\frac{|x_{min}^{(l)}|^2}{8 \left(
		|x^{(l)}_{max}|^2 S_l\mu_l^2+\zeta_{l-1}^2 \right) }.
	\right) 
	\end{aligned}
	\]
	In order to bound the patch error $\vv^{(l)}$ under the assumption that the $\supp(\vx)$ is recovered, we adopt the approach of Theorem 8 of \cite{JMLR:v18:16-505}. First
	\[
	\begin{split}
	\| \hat{\textbf{x}}^{(l)} - \textbf{x}^{(l)} \|_{2, \infty} ^{P^{(l)}}& = \max_i \| P^{(l)}(i)\left(\textbf{x}^{(l)} - \hat{\textbf{x}}^{(l)} \right)\|_2\\
	& = \sqrt{ \|\textbf{x}^{(l)}\|_{0, \infty}^{P^{(l)}}} \left( \max_i \| P^{(l)}(i)\left(\textbf{x}^{(l)} - \hat{\textbf{x}}^{(l)} \right) \|_{\infty}\right)\\
	& \leq \sqrt{ \|\textbf{x}^{(l)}\|_{0, \infty}^{P^{(l)}}} \left( \| \textbf{x}^{(l)} - \hat{\textbf{x}}^{(l)} \|_{\infty}\right).
	\end{split}
	\]
	The first equality follows from the definition of the patch norm $\| \cdot \|_{2, \infty} ^{P^{(l)}}$. The second inequality arises from the fact that for any $\textbf{z} \in \reals^{n_lM}$, with $k$ nonzeros, then  $\| \textbf{z}\|_2 \leq \sqrt{k}\| \textbf{z}\|_{\infty}$. Given that we are assuming that $\supp(\hat{\vx}^{(l)}) =\supp(\vx^{(l)})$, then the inequality on the third line follows from the fact that the largest element in a vector is at least as large as the largest element of any subset of elements of that vector. In what follows subscript notation is used to indicate the subset of entries of a vector or columns of a matrix in an index set. As $\| \textbf{x}^{(l)} - \hat{\textbf{x}}^{(l)} \|_{\infty} = \| \textbf{x}^{(l)}_{\Lambda}  - \hat{\textbf{x}}^{(l)}_{\Lambda}  \|_{\infty}$ , and recalling that the $\ell_{\infty}$ matrix norm is the maximum row sum of the absolute elements of the matrix, then
	\[
	\begin{split}
	\| \textbf{x}_{\Lambda}^{(l)} - \hat{\textbf{x}}^{(l)}_{\Lambda} \|_{\infty} & = \| (\textbf{A}_{\Lambda}D_{\Lambda})^+(\textbf{A}_{\Lambda}D_{\Lambda})\textbf{x}^{(l)}_{\Lambda} - (\textbf{A}_{\Lambda} D_{\Lambda} )^T\hat{\vx}^{(l-1)} \|_{\infty}\\
	& =  \| (\textbf{I} -(\textbf{A}_{\Lambda} D_{\Lambda} )^T(\textbf{A}_{\Lambda} D_{\Lambda} ))\textbf{x}^{(l)}_{\Lambda} - (\textbf{A}_{\Lambda} D_{\Lambda} )^T\textbf{v}^{(l-1)} \|_{\infty}\\
	& \leq \| (\textbf{I} -\textbf{A}_{\Lambda}^T \textbf{A}_{\Lambda}) \|_{\infty} \| \textbf{x}^{(l)}_{\Lambda} \|_{\infty} + \|\textbf{A}_{\Lambda}^T \textbf{v}^{(l-1)}\|_{\infty}\\
	& \leq \mu_l(\|\textbf{x}^{(l)}\|^{Q^{(l)}}_{0, \infty}-1)|x_{max}^{(l)}| + \zeta_{l-1}.
	\end{split}
	\]	
	\noindent On line one, $(\textbf{A}_{\Lambda}D_{\Lambda})^+$ denotes the Moore-Penrose inverse, the equality then follows from the fact that $\textbf{I} =  (\textbf{A}_{\Lambda}D_{\Lambda})^+(\textbf{A}_{\Lambda}D_{\Lambda})$ and from the definition of the sparse projection carried out by the forward pass of the decoder at each layer, Equation \eqref{eq:ch3_projection}. The equality on line 2 is obtained by introducing a positive and negative $\textbf{A}_{\Lambda} D_{\Lambda}\textbf{x}^{(l)}_{\Lambda}$. The inequality on the third line is obtained by applying the triangle inequality and then using the submultiplicative property of the induced matrix norm. The fourth and final inequality follows as a result of the definition of the $\ell_{\infty}$ matrix norm. The diagonal elements of $\textbf{I} -\textbf{A}_{\Lambda}^T \textbf{A}_{\Lambda}$ are all zero as $|\langle \textbf{a}_i, \textbf{a}_i \rangle | = 1 $. Off of the diagonal, at most $\|\textbf{x}^{(l)}\|^{Q^{(l)}}_{0, \infty} -1$ entries are nonzero due to the convolutional structure of $\mA$. In addition, as $|\langle \textbf{a}_i, \textbf{a}_j \rangle| \leq \mu_l$ it follows that
	\[
	\| (\textbf{I} -\textbf{A}_{\Lambda}^T \textbf{A}_{\Lambda}) \|_{\infty} \leq \mu_l (\|\textbf{x}^{(l)}\|^{Q^{(l)}}_{0, \infty} -1) \leq \mu_l S_l.
	\]
	\noindent Finally, letting $\alpha \defeq \min_{l} \{l \in [n_{l-1}M]: \; l \in \supp(\va_i)\}$ and recalling that $\|\va_i \|_2 \leq 1$, applying the Cauchy-Schwarz inequality it follows that
	\[
	\|\textbf{A}_{\Lambda}^T \textbf{v}^{(l-1)}\|_{\infty} = \max_i | \langle \textbf{a}_i, \textbf{v}^{(l-1)} \rangle| = \max_i | \langle \textbf{a}_i, P^{(l-1)}(\alpha) \textbf{v}^{(l-1)} \rangle| \leq \| P^{(l-1)}(\alpha) \vv^{l-1} \|_{2} \leq \zeta_{l-1}.
	\]
	This concludes the proof of Lemma \ref{ch3_lemma_single_layer}.
\end{proof}

\noindent  We now proceed to prove Theorem \ref{ch3_thm_prob_bound}.

\begin{proof}
	With Lemma \ref{ch3_lemma_single_layer} in place then Theorem \ref{ch3_thm_prob_bound} can be proved by induction. For the sake of convenience we let
	\[
	\gamma_l \defeq 2Mn_l\exp \left( -  \frac{|x^{(l)}_{min}|^2}{8  \left(|x^{(l)}_{max}|^2\mu_l^2S_l +\zeta_{l-1}^2 \right) } \right).
	\]
	Furthermore, and in keeping with our notation, let $Y^{(l)}$ and $\bar{Y}^{(l)}$ denote the events that the activation pathway of $\vx^{(L)}$ is recovered and not recovered up to the $l$th layer respectively, and $W^{(l)}$ and $\bar{W}^{(l)}$ be the events that the support at the $l$th layer is correctly and not correctly recovered respectively. 
	
	The base case $l=1$ follows by the construction of $\vv^{(0)}$ and by direct application of Lemma \ref{ch3_lemma_single_layer}. As a result $\prob(\bar{Y}^{(1)}) = \prob(\bar{W}^{(1)}) \leq \gamma_1$. While the bound in Lemma \ref{ch3_lemma_single_layer} was derived by conditioning on recovery at the previous layer, note that the bound still in fact applies if we condition on recovery across all preceding layers. Indeed, we could make such assumptions and still derive the same bound by using only the information concerning the layer immediately before the layer of interest. As a result
	\[
	\prob(\bar{W}^{(l)}|\cap_{k=1}^{l-1}W^{(k)}) = \prob(\bar{W}^{(l)}| Y^{(l-1)}) \leq \gamma_l
	\]
	\noindent for all $l \in [L]$. Assume now that the desired result holds true for the $l$th layer, meaning $\prob(\bar{Y}^{(l)}) \leq \sum_{k=1}^{l} \gamma_k$. Considering $\bar{Y}_{l+1}$, then
	\[
	\begin{aligned}
	\prob(\bar{Y}_{l+1}) & = \prob(\bigcup_{k=1}^{l+1} \bar{W}^{(k)})\\
	& = \prob(\bar{W}^{(l+1)} \cup \bar{Y^{(l)}})\\
	& = \prob(\bar{Y_l}) + \prob( \bar{W}^{(l+1)} \cap Y^{(l)})\\
	& = \prob(\bar{Y^{(l)}}) + \prob( \bar{W}^{(l+1)} | Y^{(l)}) \prob(Y^{(l)})\\
	& \leq \sum_{k=1}^{l} \gamma_k+ \gamma_{l+1} \prob(Y_l)\\
	& \leq \sum_{k=1}^{l} \gamma_k+ \gamma_{l+1} \\
	&= \sum_{k=1}^{l+1} \gamma_k.
	\end{aligned}
	\]
	\noindent This proves the result holds for the $l+1$th case, given that this and the base case hold true then all other cases must follow. Finally, the bound on the patch error at each layer follows immediately from Lemma \ref{ch3_lemma_single_layer}.
\end{proof}

\section{Concluding remarks} \label{ch3_sec_concluding_remarks}
Modelling the forward pass algorithm as a sparse coding problem allows us to derive recovery guarantees, which ensure that the representations computed by the forward pass are meaningful and interpretable. Our contributions in this paper are a) an approach to carrying out a probabilistic rather than worst case analysis for the recovery of activation pathways using the DCSC model, given in Definition \ref{ch3_def_DCSC_data_model}, and b) Theorem \ref{ch3_thm_prob_bound}, which extends the prior uniform bound in \cite{JMLR:v18:16-505} to one which holds with high probability. The key benefit of this result is that the proportionality of the stripe sparsity bound in regard to the dictionary coherence improves from $\mu_l^{-1}$ to $\mu_l^{-2}$ at each layer. Assuming the weight matrices are suitably conditioned, then this indicates that the forward pass algorithm is likely to recover the latent representations generated by the encoder of the DCSC for a more complex, measured in terms of the number of nonzeros per stripe, family of signals than prior work suggests. From a practical perspective, if sparse coding is an important factor explaining the efficacy of the forward pass of DCNNs, then explicitly encouraging weight matrices with low coherence during training could improve their performance.

\section*{Acknowledgements}
This work is supported by the Alan Turing Institute under the EPSRC grant EP/N510129/1 and the Ana Leaf Foundation. We would like to thank David L. Donoho, Vardan Papyan for motivating this work, Hemant Tyagi for his helpful feedback and Jeremias Sulam for stimulating discussions.

\clearpage
\bibliographystyle{unsrt}  
\bibliography{refs}

\begin{thebibliography}{1}

\bibitem{JMLR:v18:16-505}
Vardan Papyan, Yaniv Romano, and Michael Elad.
\newblock Convolutional neural networks analyzed via convolutional sparse
  coding.
\newblock {\em Journal of Machine Learning Research}, 18(83):1--52, 2017.

\bibitem{NIPS2012_4824}
Alex Krizhevsky, Ilya Sutskever, and Geoffrey~E Hinton.
\newblock Imagenet classification with deep convolutional neural networks.
\newblock In F.~Pereira, C.~J.~C. Burges, L.~Bottou, and K.~Q. Weinberger,
  editors, {\em Advances in Neural Information Processing Systems 25}, pages
  1097--1105. Curran Associates, Inc., 2012.

\bibitem{Kim2014ConvolutionalNN}
Yoon Kim.
\newblock Convolutional neural networks for sentence classification.
\newblock In {\em EMNLP}, 2014.

\bibitem{2017arXiv170102720Z}
Y.~{Zhang}, M.~{Pezeshki}, P.~{Brakel}, S.~{Zhang}, C.~L. {Yoshua Bengio}, and
  A.~{Courville}.
\newblock {Towards End-to-End Speech Recognition with Deep Convolutional Neural
  Networks}.
\newblock {\em ArXiv e-prints}, January 2017.

\bibitem{2017ITSP...65.5687P}
V.~{Papyan}, J.~{Sulam}, and M.~{Elad}.
\newblock {Working Locally Thinking Globally: Theoretical Guarantees for
  Convolutional Sparse Coding}.
\newblock {\em IEEE Transactions on Signal Processing}, 65:5687--5701, November
  2017.

\bibitem{4351958}
K.~Schnass and P.~Vandergheynst.
\newblock Average performance analysis for thresholding.
\newblock {\em IEEE Signal Processing Letters}, 14(11):828--831, Nov 2007.

\bibitem{intro_CS}
Simon Foucart and Holger Rauhut.
\newblock {\em A Mathematical Introduction to Compressive Sensing}.
\newblock Birkh\"{a}user Basel, 2013.

\bibitem{articleWelch}
L~R.~Welch.
\newblock Lower bounds on the maximum cross correlation of signals.
\newblock {\em IEEE Transactions on Information Theory}, 20:397 -- 399, 06
  1974.

\bibitem{Talagrand}
M.~Ledoux and M.~Talagrand.
\newblock {\em Probability in Banach Spaces}.
\newblock Classics in Mathematics. Springer-Verlag Berlin Heidelberg, 2011.

\end{thebibliography}
\clearpage
\appendix
\end{document}